\documentclass[]{article}

\usepackage{authblk} 

\usepackage{proceed}
 
\usepackage{times}

\usepackage{amsthm}
\usepackage{graphicx}
\usepackage{fix-cm}
\usepackage{amsmath}
\usepackage{amssymb}
\usepackage{paralist}
\usepackage{algorithmic}
\usepackage{algorithm}
\usepackage{color}

\newtheorem{thm}{Theorem}

\newtheorem{cor}{Corollary}
\newtheorem{lem}{Lemma}

\def \R {\mathbb{R}}
\def \D {\mathcal{D}}
\def \y {\mathbf{y}}

\def \E {\mathrm{E}}
\def \x {\mathbf{x}}
\def \y {\mathbf{y}}

\def \P {\mathcal{P}}
\def \c {\mathbf{c}}

\def \z {\mathbf{z}}

\def \xh {\widehat{\x}}

\def \u {\mathbf{u}}

\def \v {\mathbf{v}}

\def \B {\mathcal{B}}

\def \g {\mathbf{g}}

\def \B {\mathcal{B}}

\def \xt {\widetilde{\x}}

\def \F {\mathcal{F}}

\title{Optimal Stochastic Strongly Convex Optimization\\ with a Logarithmic Number of Projections}

\author[1]{Jianhui Chen}
\author[2]{Tianbao Yang}
\author[3]{Qihang Lin}
\author[4]{Lijun Zhang}
\author[1]{Yi Chang}
\affil[1]{Yahoo Research, Sunnyvale, CA 94089, USA}
\affil[2]{Department of Computer Science, The University of Iowa, Iowa City, IA 52242, USA}
\affil[3]{Department of Management Sciences, The University of Iowa, Iowa City, IA 52242, USA}
\affil[4]{National Key Laboratory for Novel Software Technology, Nanjing University, Nanjing 210023, China}

\begin{document}

\maketitle

\begin{abstract}

We consider stochastic strongly convex optimization with a complex inequality constraint. This complex inequality constraint may lead to computationally expensive projections in algorithmic iterations of the stochastic gradient descent~(SGD) methods. To reduce the computation costs pertaining to the projections,  
we propose an Epoch-Projection Stochastic Gradient Descent~(Epro-SGD) method. The proposed Epro-SGD method consists of a sequence of epochs; it applies SGD to an augmented objective function at each iteration within the epoch, and then performs a projection at the end of each epoch.  
Given a strongly convex optimization and for a total number of $T$ iterations, Epro-SGD requires only $\log(T)$ projections, and meanwhile attains an optimal convergence rate of $O(1/T)$, both in expectation and with a high probability. To exploit the structure of the optimization problem, we propose a proximal variant of Epro-SGD, namely Epro-ORDA, based on the optimal regularized dual averaging method. We apply the proposed methods on real-world applications; the empirical results demonstrate the effectiveness of our methods.


\end{abstract}

\section{INTRODUCTION}  \label{intro} 

Recent years have witnessed an increased interest in adopting the stochastic (sub)gradient (SGD) methods~\cite{DBLP:conf/nips/BachM11,citeulike:9912332,Nemirovski:2009:RSA:1654243.1654247} for solving large-scale machine learning  problems. In each of the algorithmic iterations, SGD reduces the computation cost by sampling one (or a small number of) example for computing a stochastic (sub)gradient. Thus the computation cost in SGD is independent of the size of the data available for training; this property makes SGD appealing for large-scale optimization. However, when the optimization problems involve a complex domain (for example a positive definite constraint or a polyhedron one), the projection operation in each iteration of SGD, which is used to ensure the feasibility of the intermediate solutions, may become the computational bottleneck.

In this paper we consider to solve the following constrained optimization problem
\begin{equation}
\begin{aligned}\label{eqn:obj0}
&\min_{\x\in\R^d}\quad f(\x)\\
&s.t. \quad c(\x)\leq 0,
\end{aligned}
\end{equation}
where $f(\x)$ is $\beta$-strongly convex~\cite{Nesterov-opt-course} and $c(\x)$ is convex. We assume a stochastic access model for $f(\cdot)$, in which the only access to $f(\cdot)$ is via a stochastic gradient oracle; in other words, given arbitrary $\x$, this stochastic gradient oracle produces a random vector $\g(\x)$, whose expectation is a subgradient of $f(\cdot)$ at the point $\x$, i.e., $\E[\g(\x)] \in \partial f(\x)$, where $\partial f(\x)$ denotes the subdifferential set of $f(\cdot)$ at $\x$. On the other hand we have the full access to the (sub)gradient of $c(\cdot)$. 

The standard SGD method~\cite{Boyd:2004:CO:993483} solves Eq.~(\ref{eqn:obj0}) by iterating the updates in Eq.~(\ref{eqn:sgd}) with an appropriate step size  $\eta_t$, e.g., $\eta_t = 1/(\beta t)$), as below
\begin{align}\label{eqn:sgd}
\x_{t+1} = \P_{\{\x\in\R^d: c(\x)\leq 0\}}\left[\x_{t} - \eta_t\g(\x_{t})\right],
\end{align}
and then returning $\widehat\x_T = \sum_{t=1}^T\x_t /T$ as the final solution for a total number of iterations $T$. Note that $\P_\D[\widehat\x]$ is a projection operator defined as 
\begin{align}
\P_\D[\widehat\x]=\arg\min_{\x\in\D}\|\x - \widehat \x\|_2^2.
\end{align} 
If the involved constraint function $c(\x)$ is complex (e.g., a polyhedral or a positive definite constraint), computing the associated projection may be computationally expensive; for example, a projection onto a positive definite cone over $\R^{d\times d}$ requires a full singular value decomposition (SVD) operation with time complexity of $O(d^3)$.


In this paper, we propose an epoch-based SGD method, called Epro-SGD, which requires only a logarithmic number of projections~(onto the feasible set), and meanwhile achieves an optimal convergence rate for stochastic strongly convex optimization. Specifically, the proposed Epro-SGD method consists of a sequence of epochs; within each of the epochs, the standard SGD is applied to optimize a composite objective function augmented by the complex constraint function, hence avoiding the expensive projections steps; at the end of every epoch, a projection operation is performed to ensure the feasibility of the intermediate solution. Our analysis shows that given a strongly convex optimization and for a total number of $T$ iterations, Epro-SGD requires only $\log(T)$ projections, and meanwhile achieves an optimal rate of convergence at $O(1/T)$, both in expectation and with a high probability. 

To exploit the structure (for example the sparisty) of the optimization problem, we propose a proximal variant of the Epro-SGD method, namely Epro-ORDA, which utilizes an existing optimal dual averaging method to solve the involved proximal mapping. Our analysis shows that Epro-ORDA similarly requires only a logarithmic number of projections while enjoys an optimal rate of convergence.

For illustration we apply the proposed Epro-SGD methods on two real-world applications, i.e., the constrained Lasso formulation and the large margin nearest neighbor~(LMNN) classification. Our experimental results demonstrate the efficiency of the proposed methods, in comparison to the existing methods.

\section{RELATED WORK}\label{sec:rw}

The present work is inspired from the break-through work in~\cite{DBLP:conf/nips/MahdaviYJZY12}, which proposed two novel one-projection-based stochastic gradient descent~(OneProj) methods for stochastic convex optimizations. Specifically the first OneProj method was developed for general convex optimization; it introduces a regularized Lagrangian function as
\begin{equation*}
L(\x, \lambda)   = f(\x) + \lambda c(\x) - \frac{\gamma}{2}\lambda^2, \quad \lambda\geq 0,
\end{equation*} 
then applies SGD to the convex-concave problem $\min_{\x\in\B}\max_{\lambda\geq 0}L(\x, \lambda)$, and finally performs only one projection at the end of all iterations, where $\B$ is a bounded ball subsuming $\mathcal F=\{\x\in\R^d: c(\x)\leq 0\}$ as a subset. 

The second OneProj method was developed for strongly convex optimization. The proposed method introduced an augmented objective function 
\begin{equation}\label{eqn:F}
F(\x) = f(\x) + \gamma \ln\left(1 + \exp\left(\frac{\lambda c(\x)}{\gamma}\right)\right),
\end{equation}
where $\gamma$ is a parameter dependent on the total number of iterations $T$, and $\lambda$ is a problem specific parameter~\cite{DBLP:conf/nips/MahdaviYJZY12}. OneProj applies SGD to the augmented objective function, specifically using a stochastic subgradient of $f(\x)$ and a subgradient of $c(\x)$, and then performs a projection step after all iterations. For a total number $T$ iterations, the OneProj method achieves a rate of convergence at $O(\log T/(\beta T))$, which is suboptimal for stochastic strongly convex optimization. 

Several recent works~\cite{DBLP:journals/jmlr/HazanK11a,conf/icml/RakhlinSS12} propose optimal methods with optimal rates of convergence at $O(1/T)$ for stochastic strongly convex optimization. In particular, the Epoch-SGD method~\cite{DBLP:journals/jmlr/HazanK11a} consists of a sequence of epochs, each of which has a geometrically decreasing step size and a geometrically increasing iteration number. This method however needs to project the intermediate solutions onto a feasible set at every algorithmic iteration; when the involved constraint is complex, the involved projection is usually computationally expensive. This limitation restricts the practical applications on large scale data analysis. Therefore we are motivated to develop an optimal stochastic algorithm for strongly convex optimization but with a constant number of projections. 

Another closely related work is the logT-SGD~\cite{logTSGD} for stochastic strongly convex and smooth optimization. LogT-SGD achieves an optimal rate of convergence, while requires to perform $O(\kappa \log_2  T)$ projections, where $\kappa$ is the ratio of the smoothness parameter to the strong convexity parameter. There are several key differences between our proposed Epro-SGD method and logT-SGD:~(i)~logT-SGD and its analysis rely on both the smoothness and the strong convexity of the objective function; in contrast, Epro-SGD only assumes that the objective function is strongly convex; (ii) the number of the required projections in logT-SGD is $O(\kappa \log_2 T)$, where  the conditional number $\kappa$ can be very large in real applications; in contrast, Epro-SGD requires at most $\log_2T$ projections. 

Besides reducing the number of projections in SGD, another line of research is based on the conditional gradient algorithms~\cite{DBLP:journals/talg/Clarkson10,Hazan:2008:SAS:1792918.1792945,thesisMJ,citeulike:11703925,Ying:2012:DML:2188385.2188386}; this type of algorithms mostly build upon the Frank-Wolfe technique~\cite{frank56}, which eschews the projection in favor of a linear optimization step; however in general, they require the smoothness assumption in the objective function. On the other hand,~\cite{arXiv:1301.4666,Hazan-free} extended Frank-Wolfe techniques to stochastic or online setting for general and strongly convex optimizations. Specifically~\cite{Hazan-free} presents an online/stochastic Frank-Wolfe (OFW) algorithm with a convergence rate $O(1/T^{1/3})$ for general convex optimization problems, which is slower than the optimal rate $O(1/\sqrt{T})$.~\cite{arXiv:1301.4666} presents an algorithm for online strongly convex optimization with an $O(\log T)$ regret bound, implying an $O(\log T/T)$ convergence rate for stochastic stronlgy convex optimization. This algorithm requires the problem domain to be a polytope, instead of a convex inequality constraint used in this paper; it also hinges on an efficient local linear optimization oracle that amounts to approximately solving a linear optimization problem over an intersection of a ball and and the feasible domain; furthermore the convergence result only holds in expectation and is sub-optimal.

\section{EPOCH-PROJECTION SGD ALGORITHM}

In this section, we present an epoch-projection SGD method, called Epro-SGD, for solving Eq.~(\ref{eqn:obj0}) and discuss its convergence result. Based on a stochastic dual averaging algorithm, we then present a proximal variant of the proposed Epro-SGD method. 

\subsection{SETUP AND BACKGROUND}

Denote the optimal solution to Eq.~(\ref{eqn:obj0}) by $\x_*$ and its domain set by $\mathcal D=\{\x\in\R^d:  c(\x)\leq 0\}$. Since $f(\x)$ is $\beta$-strongly convex~\cite{Nesterov-opt-course} and $c(\x)$ is convex, the optimization problem in Eq.~(\ref{eqn:obj0}) is strongly convex.
Note that the strong convexity in $f(\cdot)$ implies that $f(\x)\geq f(\x_*) + (\beta/2)\|\x-\x_*\|^2$ for any $\x$. Our analysis is based on the following assumptions:
\begin{enumerate}
\item[A1.] The stochastic subgradient $\g(\x)$ is uniformly bounded  by $G_1$, i.e., $\|\g(\x)\|_2\leq G_1$.
\item[A2.] The subgradient $\partial c(\x)$ is uniformly bounded by $G_2$, i.e., $\|\partial c(\x)\|_2\leq G_2$ for any $\x$.
\item[A3.] There exists a positive value $\rho>0$ such that 
\begin{eqnarray}
\left[ \min_{c(\x)=0, \, \v \in \partial c(\x), \, \v \neq 0} \|\v\|_2 \right] \geq \rho.
\end{eqnarray} 
\end{enumerate}
{\noindent \bf Remarks}~Assumptions A1 and A2 respectively impose an upper bound on the stochastic subgradient of the objective function $f(\cdot)$ and the constraint function $c(\cdot)$. Assumption A3 ensures that the projection of a point onto a feasible domain does not deviate too much from this intermediate point. Note that Assumption A1 is previously used in~\cite{DBLP:journals/jmlr/HazanK11a}; a condition similar to Assumption A3 is used in~\cite{DBLP:conf/nips/MahdaviYJZY12}, which however simply assumes that $\min_{c(\x)=0}\|\nabla c(\x)\|_2\geq \rho$, without considering possible non-differentiability in $c(\cdot)$. 

A key consequence of Assumption A3 is presented in the following lemma. 
\begin{lem}\label{lem:0}
For any $\widehat\x$, let $\widetilde\x = \arg\min_{c(\x)\leq 0}\|\x- \widehat\x\|_2^2$. If Assumption A3 holds,  then 
\begin{align}\label{eqn:keyi}
\|\xh - \xt\|_2\leq \frac{1}{\rho} [c(\xh )]_+, \quad \rho>0,
\end{align}
where $[s]_+$ is a hinge operator defined as $[s]_+=s$ if $s\geq 0$, and $[s]_+=0$ otherwise. 
\end{lem}
\begin{proof}
If $c(\xh_T)\leq 0$, we have $\widehat \x = \widetilde \x$; the inequality in Eq.~(\ref{eqn:keyi}) trivially holds. If $c(\xh_T)>0$, we can verify that $c(\xt_T)=0$, and there exists $s\geq 0$ and $\v\in\partial c(\xt_T)$ such that $\xt_T - \xh_T  + s \v =0$~(using duality theory). It follows that $\xh_T - \xt_T = s\v$~($\v\neq 0$), and thus $\xh_T - \xt_T$ is the same direction as $\v$.
It follows that 
\begin{eqnarray*}
c(\xh_T) & =  & c(\xh_T) - c(\xt_T) \geq (\xh_T - \xt_T)^{\top}\v  \\
              & =  & \|\v\|_2 \|\xh_T - \xt_T\|_2 \geq \rho \|\xh_T - \xt_T\|_2,
\end{eqnarray*}
where the last inequality uses Assumption A3. This completes the proof of this lemma.  
\end{proof}
The result in Lemma~\ref{lem:0} is closely related to {\it the polyhedral error bound condition}~\cite{DBLP:journals/mp/GilpinPS12,DBLP:journals/corr/arXiv:1512.03107}; this condition shows that the distance of a point to the optimal set of a convex optimization problem whose epigraph is a polyhedron  is bounded by the distance of the objective value at this point to the optimal objective value scaled by a constant. For illustration, we consider the optimization problem 
\begin{align*}
\min_{\x\in\R^d} \,\, [c(\x)]_+
\end{align*}
with an optimal set as $\{\x\in\R^d: c(\x)\leq 0\}$. If $c(\xh)>0$, $\xt =  \arg\min_{c(\x)\leq 0}\|\x- \widehat\x\|_2^2$ is the closest point in the optimal set to $\xh$. Therefore, by {\it the polyhedral error bound condition} of a polyhedral convex optimization, if $c(\x)$ is a polyhedral function, there exists a $\rho>0$ such that
\begin{equation*}
\|\xh - \xt\|_2\leq \frac{1}{\rho} \left( [c(\xh)]_+  - \min_\x [c(\x)]_+\right) = \frac{1}{\rho}[c(\xh)]_+.
\end{equation*}
Below we present three examples in which Assumption A3 or Lemma~\ref{lem:0} is satisfied. Example $1$: an affine function $c(\x) = \c^{\top}\x - b$ with $\rho = \|\c\|_2$. Example $2$: the $\ell_1$ norm constraint $\|\x\|_1\leq B$ where $\rho = \min_{\x: \|\x\|_1 =B}\|\partial \|\x\|_1\|_2\geq 1$.  Example $3$: the maximum  of a finite number of affine functions $c(\x) = \max_{1\leq i\leq m}\c_i^{\top}\x - b_i$ satisfying Lemma~\ref{lem:0} as well as the polyhedral error bound condition~\cite{DBLP:journals/corr/arXiv:1512.03107}. 

%
\subsection{MAIN ALGORITHM}

To solve Eq.~(\ref{eqn:obj0}) (using Epro-SGD), we introduce an augmented objective function by incorporating the constraint function as
\begin{align}\label{eqn:augment}
F(\x) = f(\x) + \lambda [c(\x)]_+.
\end{align}
It is worth noting that the augmented function in Eq.~(\ref{eqn:augment}) does not have any iteration-dependent parameter, for example the parameter $\gamma$ in Eq.~(\ref{eqn:F}). $\lambda$ is a prescribed parameter satisfying $\lambda > G_1/\rho$, as illustrated in Lemma~\ref{lem:1}.

The details of our proposed Epro-SGD algorithm is presented in Algorithm~\ref{alg:1}. Similar to Epoch-SGD~\cite{DBLP:journals/jmlr/HazanK11a}, Epro-SGD consists of a sequence of epochs, each of which has a geometrically decreasing step size and a geometrically increasing iteration number~(Line $9$ in Algorithm~\ref{alg:1}).  The updates in  every intra-epoch~(Line $5$ - $6$) are standard SGD steps applied to the augmented objective function $F(\x)$ with $\x = \x^k_t$. 
Epro-SGD is different from Epoch-SGD in that the former computes a projection only at the end of each epoch, while the latter computes a projection at each iteration. 
Consequently, when the projection step is computationally expensive~(e.g., projecting onto a positive definite constraint), Epro-SGD may require much less computation time than Epoch-SGD.
\begin{algorithm}[t]
\caption{{\small Epoch-projection SGD (Epro-SGD)}} \label{alg:1}
\begin{algorithmic}[1]
\STATE \textbf{Input}: an initial step size $\eta_1$, total number of iterations $T$, and number of iterations in the first epoch $T_1$, a Lagrangian multiplier $\lambda$~($\lambda > G_1/\rho$)
\STATE \textbf{Initialization: } $\x^1_1 \in \D$ and $k=1$
\WHILE{$\sum_{i=1}^kT_i\leq T$}
\FOR{$t = 1, \ldots, T_k$}
    \STATE Compute a stochastic gradient $\g(\x^k_t)$
    \STATE Compute $\x^k_{t+1} = \x^k_t - \eta_k(\g(\x^k_t) + \lambda \partial [c(\x^k_t)]_+)$
\ENDFOR
\STATE Compute $\widetilde\x^{k}_T=\P_{\D}[\widehat\x^k_T]$, where $\widehat\x^k_T = \sum_{t=1}^{T_k} \x^k_t/T_k$
\STATE Update $\x^{k+1}_1 = \widetilde\x^k_T$, $T_{k+1} = 2T_k, \eta_{k+1} = \eta_k/2$
\STATE Set $k=k+1$
\ENDWHILE
\end{algorithmic}
\end{algorithm}


In Lemma~\ref{lem:1}, we present an important convergence analysis for the intra-epoch steps of Algorithm~\ref{alg:1}, which are key building blocks for deriving the main results in Theorem~\ref{thm:1}.
\begin{lem}\label{lem:1}
Under Assumptions A1$\sim$A3, if we apply the update $\x_{t+1} = \x_t - \eta ( \widetilde\nabla f(\x_t;\varepsilon_t) + \lambda \nabla [c(\x_t)]_+ )$ for a number of $T$ iterations, the following equality holds
\begin{eqnarray*}
\E[f(\widetilde\x_T)]  \hskip -0.02in -  \hskip -0.02in f(\x_*)  \hskip -0.02in  \leq \hskip -0.02in  \mu \hskip -0.02in \left[\eta (G_1^2+\lambda^2G_2^2)  \hskip -0.02in + \hskip -0.02in \frac{\E[\|\x_1 - \x_*\|_2^2]}{2\eta T}\right], 
\end{eqnarray*}
where $\mu = \rho / ( \rho - G_1/\lambda)$.
\end{lem}
\begin{proof}
Let $F(\x) = f(\x) + \lambda [c(\x)]_+$ and denote by $\E_t[X]$ the expectation conditioned on the randomness until round $t-1$.  It is easy to verify that $F(\x)\geq f(\x)$, $F(\x)\geq f(\x) + \lambda c(\x)$ and $F(\x_*) = f(\x_*)$. 
For any $\x$, we have
\begin{eqnarray*}
(\x_t - \x)^{\top} \nabla F(\x_t)  &   \leq    & \frac{1}{2\eta} \left( \| \x - \x_t \|_2^2 - \|\x-\x_{t+1}\|_2^2\right) + \\
                                                &             & \frac{\eta}{2}\|\widetilde \nabla f(\x_t,\xi_t) + \lambda\nabla [c(\x_t)]_+\|_2^2 + \\
                                                &             & (\x- \x_t)^{\top}(\g(\x_t) - \nabla f(\x_t))  \\
                                                &  \leq     & \frac{1}{2\eta} \left( \|\x-\x_t\|_2^2-\|\x-\x_{t+1}\|_2^2 \right) + \\
                                                &             & \eta \left( G_1^2+ \lambda^2G_2^2 \right)  + \zeta_t(\x)\label{eqn:basic},    
\end{eqnarray*}
where $\zeta_t(\x) = (\x- \x_t)^{\top}(\g(\x_t) - \nabla f(\x_t))$. Furthermore by the convexity of $F(\x)$, we have
\begin{eqnarray*}
F(\x_t) - F(\x) & \leq & \frac{1}{2\eta}\left(\|\x-\x_t\|_2^2-\|\x-\x_{t+1}\|_2^2\right) + \\ 
                     &        &  \eta(G_1^2+ \lambda^2G_2^2)  + \zeta_t(\x).
\end{eqnarray*}
Noting that  $\E_{t}[\zeta_t(\x)] = 0$, taking expectation over randomness and summation over $t=1,\ldots, T$, we have
\begin{eqnarray*}
\frac{1}{T} \E\left[\sum_{t=1}^T(F(\x_t) - F(\x))\right]  = \E\left[F(\widehat\x_t) - F(\x)\right] \\
 \leq \frac{\E[\|\x_1 - \x\|_2^2]}{ 2\eta T} + \eta (G_1^2 + \lambda^2G_2^2). 
\end{eqnarray*}
Let $B = \E[\|\x_1 - \x_*\|_2^2] / \left(2\eta T\right) + \eta (G_1^2 + \lambda^2G_2^2)$. Since $\x_*\in\D\subseteq\B$, we have
\begin{align}\label{eqn:key}
\E\left[F(\widehat\x_t) - F(\x_*)\right] \leq B.
\end{align}
It follows that 
\begin{align}\label{eqn:B}
&\E[f(\widehat\x_T) + \lambda [c(\xh_T)]_+] \leq f(\x_*)  + B.
\end{align}
If $c(\widehat\x_T) \leq 0$, we have $\widetilde\x_T = \widehat\x_T$. Following from $F(\widetilde\x_T)\geq f(\widetilde\x_T)$ and $F(\x_*) = f(\x_*)$, we can verify that $\E [ f(\widehat\x_T) ] - f(\x_*)  \le B$ and also $\E [ f(\widehat\x_T) ] - f(\x_*)  \le \rho B / \left( \rho - G_1 / \lambda\right)$ holds. 

Next we show that $\E [ f(\widehat\x_T) ] - f(\x_*)  \le \rho B / \left( \rho - G_1 / \lambda\right)$ holds when $c(\widehat\x_T)>0$.
From Lemma~\ref{lem:0}, we have
\begin{equation} \label{eq:xx_1}
c(\widehat\x_T) \geq \rho\|\widehat\x_T - \widetilde\x_T\|_2.
\end{equation}
Moreover it follows from $\|\partial f(\x)\|_2\leq G_1$ and $f(\x_*)\leq f(\xt_T)$ that the following inequality holds
\begin{eqnarray} \label{eq:xx_2}
f(\x_*) - f(\widehat\x_T) &\leq &f(\x_*) - f(\widetilde\x_T) + f(\widetilde\x_T) - f(\widehat\x_T)  \nonumber \\
                                     & \leq & G_1\|\widehat\x_T - \widetilde\x_T\|_2.
\end{eqnarray}
Substituting Eqs.~(\ref{eq:xx_1}) and~(\ref{eq:xx_2}) into Eq.~(\ref{eqn:B}), we have
\begin{eqnarray*}
\lambda\rho \E[\|\widehat\x_T - \widetilde\x_T\|_2] & \leq & \E[f(\x_*) - f(\widehat\x_T)] + B \\ 
                                                                                & \leq & G_1\E[\|\widehat\x_T - \widetilde\x_T\|_2] + B.
\end{eqnarray*}
By some rearrangement, we have $\E[\|\widehat\x_T - \widetilde\x_T\|_2] \leq {B} / \left( \lambda\rho - G_1\right)$.
Furthermore we have 
\begin{eqnarray*}
\E[f(\widetilde\x_T)] - f(\x_*)  \leq \E[f(\widetilde\x_T) - f(\widehat\x_T)] + \E[f(\widehat\x_T)] \\ 
- f(\x_*)  
\leq \E[G_1 \|\widehat\x_T - \widetilde\x_T\|_2] + B\leq \frac{\lambda \rho}{\lambda\rho - G_1} B,
\end{eqnarray*}
where the second inequality follows from $\|\nabla f(\x)\|_2 \leq \E\|\nabla f(\x;\varepsilon)\| \leq G_1$, and $|f(\x) - f(\y)|\leq G_1\|\x-\y\|_2$ for any $\x,\y$. This completes the proof of the lemma. 
\end{proof}
We present a main convergence result of the Epro-SGD algorithm in the following theorem.
\begin{thm}\label{thm:1}
Under Assumptions A1$\sim$A3 and given that $f(\x)$ is $\beta$-strongly convex, if we let $\mu = \rho/(\rho- G_1/\lambda)$, $G^2 = G_1^2 + \lambda^2 G_2^2$,  and set $T_1=8, \eta_1=\mu/(2\beta)$, the total number of epochs $k^\dagger$ in Algorithm~\ref{alg:1} is given by
\begin{equation} \label{thm1:eq1}
k^\dagger = \left\lceil \log_2\left(\frac{T}{8} + 1\right) \right\rceil\leq \log_2 \left(\frac{T}{4} \right), 
\end{equation}
the solution $\x^{k^\dagger+1}_1$ enjoys a convergence rate of
\begin{equation}  \label{thm1:eq2}
\E[f(\x^{k^\dagger+1}_1)] - f(\x_*) \leq \frac{32\mu^2 G^2}{\beta (T+8)},
\end{equation}
and $c(\x_1^{k^\dagger+1})\leq 0$. 
\end{thm}
\begin{proof}
From the updating rule $T_{k+1} = 2  T_k$, we can easily verify Eq.~(\ref{thm1:eq1}). Since $\x_1^{k^\dagger + 1} = \widetilde \x_T^{k^\dagger} \in \D$, the inequality $c(\x_1^{k^\dagger+1})\leq 0$ trivially holds. 

Let $V_k = \mu^2 G^2 / \left( 2^{k-2}\beta \right)$. It follows that $T_k  = 2^{k+2} = 16\mu^2G^2 / \left( V_k \beta \right)$ and $\eta_k = \mu / \left( 2^k \beta \right) =  V_k / \left( 4\mu G^2 \right)$. 
Next we show the inequality
\begin{equation} \label{thm1:eq3}
\E[f(\x^k)] - f(\x_*) \leq V_k
\end{equation}
holds by induction. Note that Eq.~(\ref{thm1:eq3}) implies $\E[f(\x^{k+1}_1)] - f(\x_*) \leq 32 \mu^2 G^2 / \left(\beta (T+8)\right)$, due to $V_k < 32 \mu^2 G^2 / \left(\beta (T+8)\right)$. Let $\Delta_k = f(\x^k_1) - f(\x_*)$. It follows from Lemma~\ref{lem:3}, $\mu>1$, and $G^2>G_1^2$, the inequality in Eq.~(\ref{thm1:eq3}) holds when $k=1$. Assuming that Eq.~(\ref{thm1:eq3}) holds for $k = k^\dagger$, we show that Eq.~(\ref{thm1:eq3}) holds for $k = {k^\dagger}_1 + 1$. 

For a random variable $X$ measurable with respect to the randomness up to epoch $k^\dagger+1$. Let $\E_{k^\dagger}[X]$ denote the expectation conditioned on all the randomness up to epoch $k^\dagger$. Following Lemma~\ref{lem:1}, we have
\[
\E_{k^\dagger}[\Delta_{{k^\dagger}+1}] \leq \mu \left[\eta_{k^\dagger} G^2+ \frac{\E[\|\x^{k^\dagger}_1 - \x_*\|_2^2]}{2\eta_{k^\dagger} T_{k^\dagger}}\right].
\]
Since $\Delta_{k^\dagger} = f(\x^{k^\dagger}_1) - f(\x_*)\geq \beta\|\x_1^{k^\dagger}-\x_*\|_2^2/2$ by the strong convexity in $f(\cdot)$, we have
\begin{eqnarray*}
\E[\Delta_{{k^\dagger}+1}]  \hskip -0.08in &   \leq  & \hskip -0.08in   \mu \left[\eta_{k^\dagger} G^2+ \frac{\E[\Delta_{k^\dagger}]}{\eta_{k^\dagger} T_{k^\dagger} \beta}\right] \\
                                            \hskip -0.08in & =  & \hskip -0.08in \mu \eta_{k^\dagger} G^2 + \frac{V_{k^\dagger} \mu}{\eta_{k^\dagger} T_{k^\dagger} \beta} = \frac{V_{k^\dagger}}{4} + \frac{V_{k^\dagger}}{4} = V_{k^\dagger + 1},
\end{eqnarray*}
which completes the proof of this theorem. 
\end{proof}
\textbf{Remark}~We compare the obtained main results in Theorem~\ref{thm:1} with several existing works. Firstly Eq.~(\ref{thm1:eq2}) implies that Epro-SGD achieves an optimal bound $O(1/T)$, matching the lower bound for a strongly convex problem~\cite{DBLP:journals/jmlr/HazanK11a}. Secondly in contrast to the OneProj method~\cite{DBLP:conf/nips/MahdaviYJZY12} with a convergence rate $O(\log T/T)$, Epro-SGD uses no more than $\log_2(T/4)$ projections to obtain an $O(1/T)$ convergence rate. Epro-SGD thus has better control over the solution for not deviating (too much) from the feasibility domain in the intermediate iterations. 
Thirdly compared to Epoch-SGD with its convergence rate bounded by $O\left( 8G_1^2 / \left(\beta T\right)\right)$, the convergence rate bound of Epro-SGD is only worse by a factor of constant $4\mu^2G^2/G_1^2$.  Particularly consider a positive definite  constraint with $\rho=1$, $\mu = 2$, and $\lambda = 2G_1/\rho$, we have $G^2 = 5G_1^2$ and the bound of Epro-SGD is only worse by a factor of $80$ than Epoch-SGD. Finally compared to the logT-SGD algorithm~\cite{logTSGD} which requires $O(\kappa \log_2 T)$ projections~($\kappa$ is the conditional number),  the number of projections in Epro-SGD is independent of  the conditional number. 
The main results in Lemma~\ref{lem:1} and Theorem~\ref{thm:1} are expected convergence bounds. In Theorem~\ref{cor:1}~(proof provided in Appendix) we show that Epro-SGD also enjoys a high probability bound under a boundedness assumption, i.e., $\|\x_* -\x_t\|_2\leq D$ for all $t$. Note that the existing Epoch-SGD method~\cite{DBLP:journals/jmlr/HazanK11a} uses two different methods to derive its high probability bounds. Specifically the first method relies on an efficient function evaluator to select the best solutions among multiple trials of run; while the second one modifies the updating rule by projecting the solution onto the intersection of the domain and a center-shifted bounded ball with decaying radius. These two methods however may lead to additional computation steps, if being adopted for deriving high probability bounds for Epro-SGD. 
\begin{thm}\label{cor:1}
Under Assumptions A1$\sim$A3 and given $\|\x_t-\x_*\|_2\leq D$ for all $t$. If we let $\mu = \rho/(\rho- G_1/\lambda)$, $G^2 = G_1^2 + \lambda^2 G_2^2, C =\left( 8G_1^2 / \beta + 2G_1D\right) \ln (m/\epsilon)  + 2G_1D$, and set   $T_1\geq \max\left( 3C\beta / \left(\mu G^2 \right), 9\right)$, $\eta_1=\mu/(3\beta)$, the total number of epochs $k^\dagger$ in Algorithm~\ref{alg:1} is given by 
\[
k^\dagger=\left \lfloor\log_2\left(\frac{T}{T_1}+1\right)\right\rfloor\leq \log_2(T/4),
\]
and the final solution $\x^{k^\dagger+1}_1$ enjoys a convergence rate of
\[
f(\x^{k^\dagger+1}_1) - f(\x_*) \leq \frac{4T_1\mu^2 G^2}{\beta (T+T_1)}
\]
with a probability at least $1-\delta$, where $m=\lceil 2\log_2 T\rceil$.
\end{thm}
\textbf{Remark}~The assumption $\|\x_* -\x_t\|_2\leq D$ can be satisfied practically, if we estimate the value of $D$ such that $\|\x_*\|_2\leq D/2$, and then project the intermediate solutions onto $\|\x\|_2\leq D/2$ at every iteration. Note that Epoch-SGD~\cite{DBLP:journals/jmlr/HazanK11a} requires a total number of $T$ projections, and its high probability bound of Epoch-SGD is denoted by $\displaystyle f(\x^{k^\dagger+1}_1) - f(\x_*)\leq 1200G_1^2\log (1/\widetilde \delta) / \left( \beta T \right)$ with a probability at lest $1-\delta$, where $\widetilde\delta = \delta / \left( \left\lfloor \log_2(T/300 + 1)\right\rfloor \right)$.

\subsection{A PROXIMAL VARIANT}

We propose a proximal extension of Epro-SGD, by exploiting the structure of the objective function. Let the objective function in Eq.~(\ref{eqn:obj0}) be a sum of two components
\[
\widehat f(\x) = f(\x) + g(\x),
\] where $g(\x)$ is a relatively simple function, for example a squared $\ell_2$-norm or $\ell_1$-norm, such that the involved proximal mapping
\[
\min_{\x\in\R^d}\quad g(\x) + \frac{1}{2}\|\x - \xh\|_2^2
\]
is easy to compute. The optimization problem in Eq.~(\ref{eqn:obj0}) can be rewritten as 
\begin{equation}\label{eqn:prox}
\begin{aligned}
&\min_{\x\in\R^d}\quad f(\x) + g(\x)\\
& s.t.\quad c(\x)\leq 0.
\end{aligned}
\end{equation}
Denote by $\x_*$ the optimal solution to Eq.~(\ref{eqn:prox}). We similarly introduce an augmented objective function as
\begin{align}\label{eqn:new}
F(\x) = f(\x) + \lambda [c(\x)]_+ + g(\x).
\end{align}
The update of the proximal SGD method for solving~(\ref{eqn:prox})~\cite{DBLP:conf/colt/DuchiSST10,DBLP:conf/nips/DuchiS09,RePEc:cor:louvco:2007076}  is given by 
\begin{align}\label{eqn:prix}
\x_{t+1} =\arg\min_{\x\in\D} \frac{1}{2}\|\x - (\x_t - \eta\g(\x_t) )\|_2^2 + \eta g(\x).
\end{align}
If $g(\x)$ is a sparse regularizer, the proximal SGD can guarantee the sparsity in the intermediate solutions and usually yields better convergence than the standard SGD. However, given a complex constraint, solving the proximal mapping may be computational expensive. 
Therefore, we consider a proximal variant of Epro-SGD which involves only the proximal mapping of $g(\x)$ without the constraint $\x\in\D$. An instinctive solution is to use the following update in place of step 6 in Algorithm~\ref{alg:1}:
\begin{eqnarray}\label{eqn:cb}
\x^k_{t+1} & \hskip -0.12in = \hskip -0.12in &  \arg\min_{\x\in\R^d} \frac{1}{2} \| \x - \nonumber \\
                 & \hskip -0.25in                        &  \left[ \x^k_t \hskip -0.03in - \hskip -0.03in \eta_k(\g(\x^k_t) \hskip -0.03in + \hskip -0.03in \lambda \partial [c(\x^k_t)]_+) \right] \hskip -0.03in \|_2^2  \hskip -0.03in +  \hskip -0.03in \eta_k g(\x).
\end{eqnarray}
Based on this update and using techniques in Lemma~\ref{lem:1}, we obtain a similar convergence result~(proof provided in Appendix), as presented in the following lemma~\cite{duchi-2009-efficient}. 
\begin{lem}\label{lem:prix}
Under Assumptions A1$\sim$A3 and setting $\mu = \rho / \left( \rho - G_1/\lambda\right)$, by applying the update in Eq.~(\ref{eqn:cb}) a number of $T$ iterations, we have
\begin{eqnarray} \label{lemma3:eq0}
\E[\widehat f(\widetilde\x^k_T)] - \widehat f(\x_*)  & \leq & \mu \E\left[\eta G^2+ \frac{\|\x^k_1 - \x_*\|_2^2}{2\eta T} \right . \nonumber \\
                                                              &         & \left.+ \frac{g(\x^k_1)-g(\x^k_{T+1})}{T}\right],
\end{eqnarray}
where $G^2=(G_1^2+\lambda^2G_2^2)$, and $\widetilde\x^k_T$ denotes the projected solution of the averaged solution $\widehat\x^k_T = \sum_{t=1}^T\x^k_t / T$. 
\end{lem}
Different from the main result in Lemma~\ref{lem:1}, Eq.~(\ref{lemma3:eq0}) has an additional term $(g(\x^k_1) - g(\x^k_{T+1}))/T_k$; it makes the convergence analysis in Epro-SGD difficult. To overcome this difficulty, we adopt the optimal regularized dual averaging (ORDA) algorithm~\cite{NIPS2012_4543} for solving Eq.~(\ref{eqn:new}). The details of ORDA are presented in Algorithm~\ref{alg:ORDA}. The main convergence results of ORDA are summarized in the following lemma~(proof provided in Appendix). 
\begin{algorithm}[t]
\caption{{\small Optimal Regularized Dual Averaging~(ORDA)}} \label{alg:ORDA}
\begin{algorithmic}[1]
\STATE \textbf{Input}: a step size $\eta$,  the  number iterations $T$, and the initial solution $\x_1$,
\STATE Set $\theta_t = \frac{2}{t+1}$, $\nu_t = \frac{2}{t}$, $\gamma_t =\frac{ t^{3/2}}{\eta}$ and $\z_1 = \x_1$
\FOR{$t=1,\ldots, T+1$}
    \STATE compute $\u_t = (1-\theta_t)\x_t + \theta_t\z_t$
    \STATE compute a stochastic subgradient   $\g(\x_t)$ of $f(\x)$ at $\x_t$ and a subgradient of $[c(\x_t)]_+$
    \STATE let $\bar\g_t= \theta_t\nu_t\left(\sum_{\tau=1}^t\frac{\g(\x_\tau) + \lambda\partial [c(\x_\tau)]_+}{\nu_\tau}\right)$
    \STATE compute 
$\z_{t+1}= \arg\min_{\x} \bar\g_t^{\top}\x +\frac{ \theta_t\nu_t\gamma_{t+1}}{2}\|\x - \x_1\|_2^2 + g(\x) $
    \STATE compute 
$\x_{t+1}= \arg\min_{\x}\x^{\top}(\g(\x_t) + \lambda\partial [c(\x_t)]_+)   +\frac{\gamma_{t}}{2}\|\x - \u_t\|_2^2+ g(\x)$
   \ENDFOR
\STATE \textbf{Output}:  $\xh_T = \x_{T+2}$
\end{algorithmic}
\end{algorithm}
\begin{lem}\label{lem:ORDA}
Under Assumptions A1$\sim$A3 and setting $\mu = \rho / \left( \rho - G_1/\lambda\right)$, by running ORDA a number of $T$ iterations for solving the augmented objective~(\ref{eqn:new}), we have
\begin{align*}
\E[F(\xh_T) -  F(\x_*)]\leq \frac{4 \|\x_1 - \x_*\|_2^2}{\eta\sqrt{T}} + \frac{2\eta (3G_1 + 2\lambda G_2)^2}{ \sqrt{T}},
\end{align*}
and 
\begin{eqnarray*}
\E[\widehat f(\widetilde\x_T)] - \widehat f(\x_*) & \leq  & 
\mu \E\left[\frac{4 \|\x_1 - \x_*\|_2^2}{\eta\sqrt{T}}  + \right. \\ 
                                                          &         & \left .  \frac{2\eta (3G_1 + 2\lambda G_2)^2}{ \sqrt{T}} \right],
\end{eqnarray*}
where $\widetilde\x_T$ denotes the projected solution of the final  solution $\widehat\x_T$.
\end{lem}
\begin{algorithm}[t]
\caption{{\small Epoch-projection ORDA (Epro-ORDA)}}  \label{alg:Epro-ORDA}
\begin{algorithmic}[1]
\STATE \textbf{Input}: an initial step size $\eta_1$, total number of iterations $T$, and number of iterations in the first epoch $T_1$, a Lagrangian multiplier $\lambda > G_1/\rho$
\STATE \textbf{Initialization: } $\x^1_1 \in \D$ and $k=1$.
\WHILE{$\sum_{i=1}^kT_i\leq T$}
\STATE Run ORDA to obtain $\xh_T^k = \text{ORDA}(\x_1^k, \eta_k, T_k)$
\STATE Compute $\widetilde\x^{k}_T=\P_{\D}[\widehat\x^k_T]$
\STATE Update $\x^{k+1}_1 = \widetilde\x^k_T$, $T_{k+1} = 2T_k, \eta_{k+1} = \eta_k/\sqrt{2}$
\STATE Set $k=k+1$
\ENDWHILE
\end{algorithmic}
\end{algorithm}
We present a proximal variant of Epro-SGD, namely Epro-ORDA, in Algorithm~\ref{alg:Epro-ORDA}, and summarize its convergence results in Theorem~\ref{thm:2}. Note that Algorithm~\ref{alg:ORDA} and the convergence analysis in Lemma~\ref{lem:ORDA} are independent of the strong convexity in $\widehat f(\x)$; the strong convexity is however used for analyzing the convergence of Epro-ORDA in Theorem~\ref{thm:2}~(proof provided in Appendix).
\begin{thm}\label{thm:2}
Under Assumptions A1$\sim$A3 and given that $\widehat f(\x)$ is $\beta$-strongly convex, if we let $\mu = \rho/(\rho- G_1/\lambda)$ and $G = 3G_1 + 2\lambda G_2$,  and set $T_1=16, \eta_1=\mu/\beta$, then the total number of epochs $k^\dagger$ in Algorithm~\ref{alg:Epro-ORDA} is given by
\[
k^\dagger = \left\lfloor \log_2\left(\frac{T}{17} + 1\right) \right\rfloor\leq \log_2(T/8),
\]
and the final solution $\x^{k^\dagger+1}_1$ enjoys a convergence rate of
\[
\E[\widehat f(\x^{k^\dagger+1}_1)] - \widehat f(\x_*) \leq \frac{68\mu^2 G^2}{\beta (T+17)},
\]
and $c(\x_1^{k^\dagger+1})\leq 0$.
\end{thm}

\section{AN EXAMPLE OF SOLVING LMNN VIA EPRO-SGD}\label{sec:lmnn}

In this section, we discuss an application of applying the proposed Epro-SGD to solve a high dimensional distance metric learning (DML) with a  large margin formulation, i.e., the large margin nearest neighbor (LMNN) classification method~\cite{weinberger2009}. LMNN classification is one of the state-of-the-art methods for k-nearest neighbor classification. It learns a positive semi-definite distance metric, based on which the examples from the k-nearest neighbors always belong to the same class, while the examples from different classes are separated by a large margin. 

To describe the LMNN method, we first present some notations. Let $(x_i, y_i), i=1, 2, \cdots, \widehat N$, be a set of data points, where $x_i\in\R^d$ and $y\in\mathcal Y$ denote the feature representation and the class label, respectively. Let $A$ be a positive definite matrix that defines a distance metric as $\mbox{dist}(x_1, x_2)=\|x_1-x_2\|_A^2=(x_1-x_2)^{\top}A(x_1-x_2)$. To learn a distance metric that separates the examples from different classes by a large margin, one needs to extract a set of similar examples (from the same class) and dissimilar examples~(from a different class), denoted by $(x^j_1,x^j_2,x^j_3), j=1,\ldots N$, where $x^j_1$ shares the same class label to $x^j_2$ and a different class from $x^j_3$.  To this end, for each example $x^j_1=x_i$ one can form $x^j_2$ by extracting the k nearest neighbors (defined by an Euclidean distance metric) that share the same class label to $x_i$, and form $x^j_3$ by extracting a set of examples that have a different class label. Then an appropriate distance metric could be obtained from the following constrained optimization problem  
\begin{eqnarray} \label{eq:lmnn}
\min_{A}        &&   \hskip -0.25in \frac{c}{N}\sum_{j=1}^N\ell\left(A, x^j_1,x^j_2,x^j_3\right) \hskip -0.04in + \hskip -0.04in (1-c)tr(AL) \hskip -0.04in + \hskip -0.04in \frac{\mu_1}{2}\|A\|_F^2 \nonumber \\
s.t.                &&   \hskip -0.25in A\succeq \epsilon I,
\end{eqnarray}
where $\ell (A, x^j_1,x^j_2,x^j_3) =\max(0, \|x^j_1-x^j_2\|_A^2 -\|x^j_1-x^j_3\|_A^2+1)$ is a hinge loss and $c\in(0, 1)$ is a trade-off parameter. In Eq.~(\ref{eq:lmnn}), $A\succeq \epsilon I$ is used as the constraint to ensure that Assumption A3 holds. Minimizing the first term is equivalent to maximizing the margin between $\|x^j_1-x^j_3\|_A^2$ and $\|x^j_1-x^j_2\|_A^2$. The matrix $L$ encodes certain prior knowledge about the distance metric; for example, the original LMNN work~\cite{weinberger2009} defines $L$ as $L=\sum_{l=1}^m\|x^l_1-x^l_2\|_A^2/m$, where $(x^l_1,x^l_2)$ are all k-nearest neighbor pairs from the same class. Other works~\cite{conf/aaai/LiuTTL10} have used a weighted summation of distances between all data pairs $L=\sum_{i\neq j}^nw_{ij}\|x_i-x_j\|_A^2/n(n-1)$ or intra-class covariance matrix~\cite{Qi:2009:ESM:1553374.1553482}. The last term $\|A\|_F^2/2$ is used as a regularization term and also makes the objective function strongly convex.

For data sets of very high dimensionality, i.e., $d \gg n$, LMNN in Eq.~(\ref{eq:lmnn}) usually produces a sub-optimal solution~\cite{Qi:2009:ESM:1553374.1553482}, as this formulation does not capture the sparsity structure of the features. Therefore we add a sparse regularizer and express the formulation below
\begin{eqnarray}\label{eq:lmnn-sparse}
\min_{A} &&\frac{c}{N}\sum_{j=1}^N\ell\left(A, x^j_1,x^j_2,x^j_3\right) + (1-c)tr(AL)      \nonumber \\ 
                                          &&    +  \frac{\mu_1}{2}\|A\|_F^2 + \mu_2 \|A\|^{\text{off}}_1 \nonumber \\
s.t.                                     &&    A\succeq \epsilon I,
\end{eqnarray}
where $\|A\|^{\text{off}}_1=\sum_{i\neq j}|A_{ij}|$ is an elmenent-wise $\ell_1$-norm excluding the diagonal entries. Note that this sparse regularizer $\|A\|^{\text{off}}_1$ have been previously used in~\cite{Qi:2009:ESM:1553374.1553482} for a different purpose. 

Many standard optimization solvers or algorithms may not be efficient for solving Eq.~(\ref{eq:lmnn-sparse}). Firstly, the optimization problem in Eq.~(\ref{eq:lmnn-sparse}) can be formulated as a semi-definite program~(SDP); however, general SDP solvers usually scale poorly with the number of triplets and is not suitable for large scale data analysis. Secondly, the gradient decent method presented in~\cite{weinberger2009} requires to project intermediate solutions onto a positive definite cone; this operation invokes expensive singular value decomposition~(SVD) for a large matrix and this limitation restricts the real-world applications of the gradient descent method. 
Thirdly,~\cite{Qi:2009:ESM:1553374.1553482} employs a block coordinate descent~(BCD) method to solve an L$1$-penalized log-det optimization problem; the BCD method is not suitable for solving Eq.~(\ref{eq:lmnn-sparse}), as the loss function is not linear in the variable $A$.

We employ the proposed Epro-SGD algorithm to solve the LMNN formulation in Eq.~(\ref{eq:lmnn-sparse}). Let $f(A) = \frac{c}{N}\sum_{j=1}^N\ell(A, x^j_1,x^j_2,x^j_3) + (1-c)tr(AL)$ and $g(A) = \frac{\mu_1}{2}\|A\|_F^2 + \mu_2 \|A\|^{\text{off}}_1$. The positive definite constraint can be rewritten into an inequality constraint as $c(A) = \epsilon -\lambda_{\min}(A)\leq 0$, where $\lambda_{\min}(\cdot)$ denotes the minimum eigen-value of the matrix $A$. We also make the correspondences $\R^d\rightarrow \R^{d\times d}$, $\x\rightarrow A$, $\|\x\|_2\rightarrow \|A\|_F$, and provide necessary details below in a question-answer form
\begin{itemize}
\item How to compute the stochastic gradient of $f(A)$? First sample one triplet $(x^j_1,x^j_2,x^j_3)$ (or a small number of triplets) and then compute $\widetilde\nabla f(A;\varepsilon) = c[(x^j_1-x^j_2)(x^j_1 -x^j_2)^{\top}- (x^j_1-x^j_3)(x^j_1 -x^j_3)^{\top}]  + (1-c)L$ if $\ell(\|x^j_1-x^j_2\|_A^2 -\|x^j_1-x^j_3\|_A^2+1)>0$, $\widetilde\nabla f(A;\varepsilon) = (1-c)L$ otherwise.
\item How to compute the gradient of $[c(A)]_+=[\epsilon-\lambda_{\min}(A)]_+$? By the theory of matrix analysis, the subgradient of $[c(A)]_+$ can be computed  by $\partial c(A)= -\u\u^{\top}$ if $c(A)>0$, and zero otherwise,  where $\u$ denotes the eigevector of $A$ associated with its minimum eigenvalue.
\item What is the solution to the following proximal gradient step?
\begin{equation*}
\min_{A} \frac{1}{2}\|A - \bar A_{t+1}\|_F^2+ \eta \left(\frac{\mu_1}{2}\|A\|_F^2  
+ \mu_2 \|A\|^{\text{off}}_1\right).
\end{equation*}
The solution can be obtained via a soft-thresholding algorithm~\cite{Beck:2009:FIS:1658360.1658364}.
\item What are the appropriate values for $\beta, \rho, r, \lambda$, that are necessary for running the algorithm? The value of $\beta=\mu_1$. The value of $\rho$ is  $ \min_{c(A)=0}\|\nabla c(A)\|_F = 1$.  The value of $r$ can be set to $\sqrt{2c/\mu_1}$. The value of $G_2=1$.  The value of $G_1$ can be estimated as $8cR^2 + (1-c)\|L\|_F+\mu_1r + \mu_2d$ if we assume $\|x_i\|_2\leq R, i=1,\ldots, n$. The value of $\lambda> G_1$ is usually tuned among a set of prespecified values.
\end{itemize}
Finally, we discuss the impact of employing Epro-SGD and Epro-ORDA method on accelerating the computation for solving LMNN. Note that at each iteration to compute the gradient of $c(A)$, we need to compute the minimum eigen-value and its eigen-vector. For a dense matrix, it usually involves a time complexity of $O(d^2)$. However, by employing a proximal projection, we can guarantee that the intermediate solution $A_t$ is a element-wise sparse solution, for which the computation of the last eigen-pair can be substantially  reduced to be  linear to the  number of non-zeros elements in $A_t$.

To analyze the running time compared to the Epoch-SGD method, let us assume we are interested in an $\epsilon$-accurate solution. In the following discussion, we take a particular choice of $\lambda = 2G_1$ and  suppress the dependence on constants and only consider dependence on $T$, $G_1$ and $d$. 
The number of iterations required by Epro-SGD or Epro-ORDA is $\Omega\left( G_1^2 / \epsilon\mu_1 \right)$. Taking into account the running time per iteration, the total running time of Epoch-SGD is $\Omega \left( G^2_1d^3 / \left(\epsilon \mu_1 \right) \right)$ and that of Epro-SGD/Epro-ORDA  is $\Omega \left( G^2_1d^2 / \left(\epsilon \mu_1 \right) \right)$. When $d$ is very large, the speed-up can be orders of magnitude.


\section{EXPERIMENTS}

In this section, we empirically demonstrate the efficiency and effectiveness of the proposed Epro-SGD algorithm. We compare the following four algorithms:
\begin{itemize}
\item Stochastic sub-Gradient Descent method (SGD)~\cite{Shalev-Shwartz:2007:PPE:1273496.1273598}: we set the step size $\eta_t = 1 / (\lambda t)$ and SGD achieves a rate of convergence $O(logT/T)$, requiring $O(T)$ projections for a constrained convex optimization problem.
\vskip -0.1in
\item One-Projection SGD method (OneProj)~\cite{DBLP:conf/nips/MahdaviYJZY12}: we set the step size $\eta_t = 1 / (\lambda t)$ and OneProj achieves a rate of convergence $O(logT/T)$, requiring only one projection for a constrained strongly convex optimization problem. 
\item $O(logT)$-projections SGD method (logT)~\cite{logTSGD}: we set the step size $\eta_t = 1 / (\sqrt{6} L)$ and logT achieves a rate of convergence $O(1/T)$, requiring $O(\log T)$ projections steps for a constrained strongly convex optimization problem. 
\item the proposed Epro-SGD with $O(logT)$ number of projections (Epro): we set the step size $\eta_t = 1 / (\lambda t)$ and Epro archives a rate of convergence $O(1/T)$ with $O(\log T)$ projections steps for a constrained strongly convex optimization problem.
\end{itemize}
For illustration, we apply the competing algorithms for solving the constrained Lasso problem and the Large Margin Nearest Neighbor Classification (LMNN) in Eq.~(\ref{eq:lmnn-sparse}) respectively. We implement all algorithms using Matlab R$2015$a and conduct all simulations on an Intel(R) Xeon(R) CPU E5-2430~(15M Cache, 2.20 GHz). 


\subsection{EXPERIMENTS ON THE CONSTRAINED LASSO FORMULATION} 

\begin{figure}
\includegraphics[width=0.24\textwidth,height=3cm]{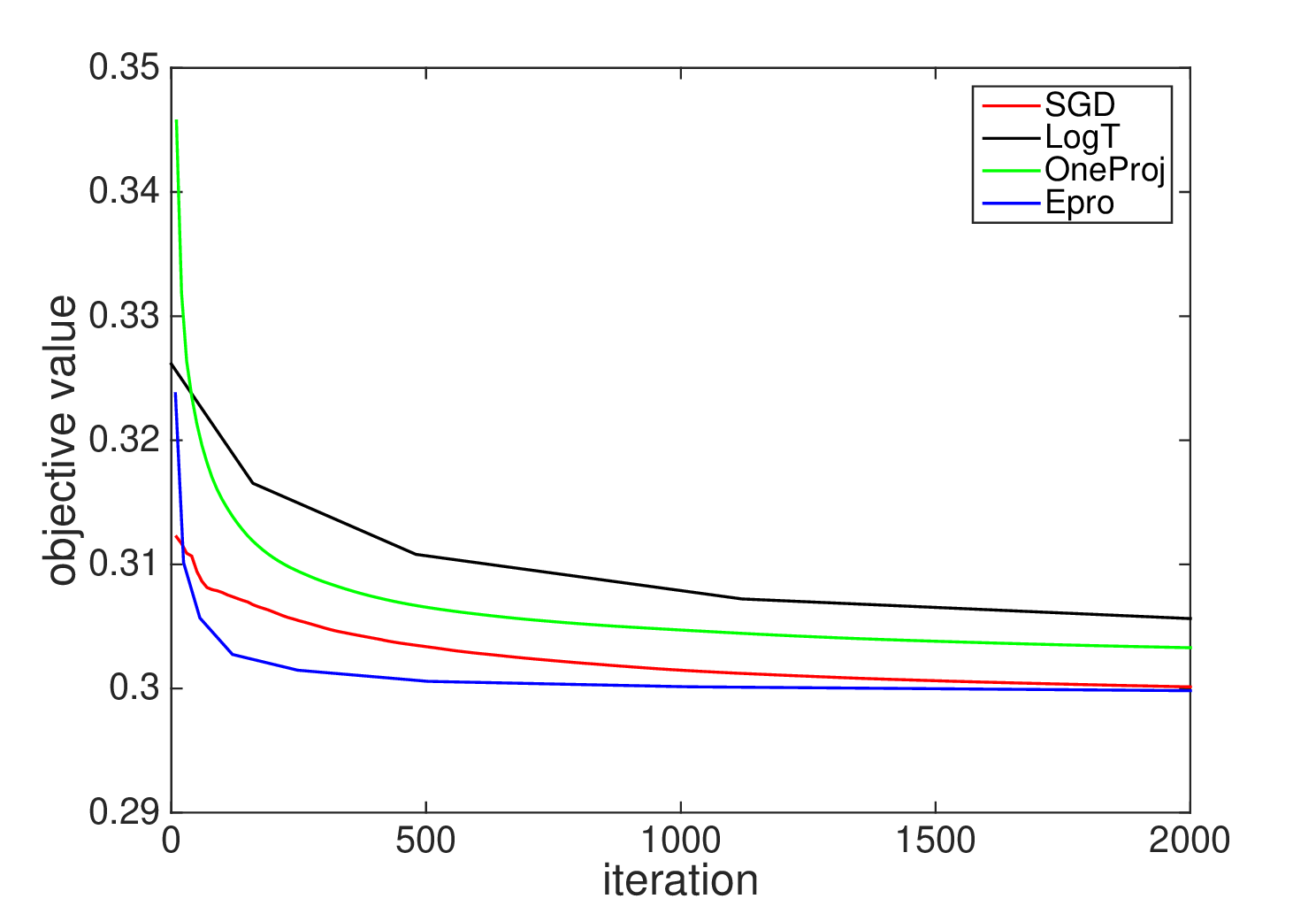}  \hskip -0.1in
\includegraphics[width=0.24\textwidth,height=3cm]{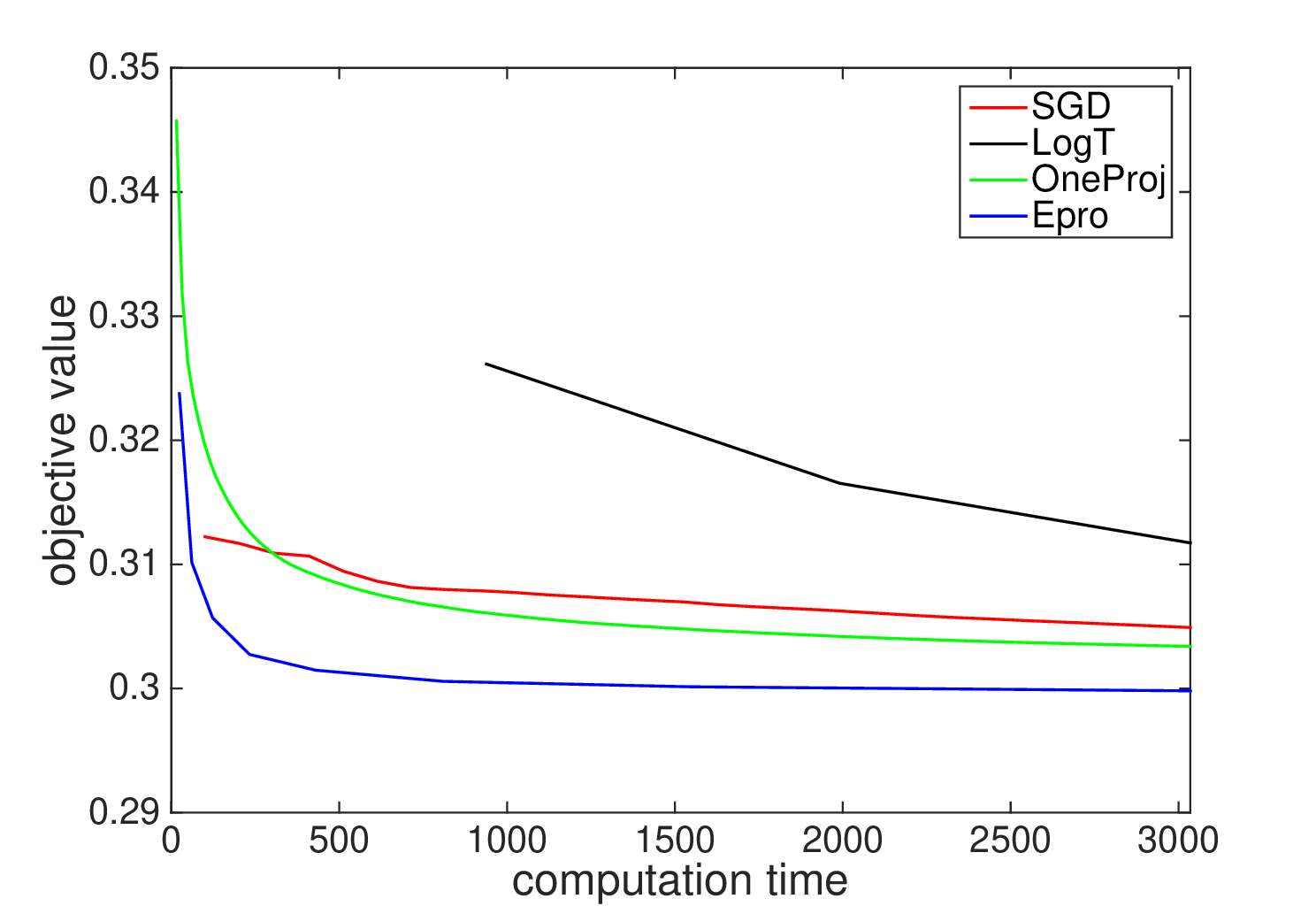}
\caption{Empirical comparison of the four competing methods for solving Eq.~(\ref{eq:lasso}). (1) Left plot: the change of the objective values with respect to the iteration number. (2) Right plot: the change of the objective values with respect to the computation time (in seconds).}
\label{fig:lasso}       
\end{figure}
We apply the proposed Epro-SGD algorithm and the other three competing algorithms to solve the L$1$-norm constrained least squares optimization problem
\begin{eqnarray} \label{eq:lasso}
\min_w        &&  \frac{1}{2 N} \sum_{i=1}^N \left( x_i^T w - y_i \right)^2 + \alpha \|w\|^2 \nonumber \\
s.t.              &&  \| w \|_1 \le \beta. 
\end{eqnarray}
Eq.~(\ref{eq:lasso}) is an equivalent constrained counterpart of the well studied Lasso formation~\cite{lasso}. They aim at achieving entry-wise sparsity in the weight vector $w$ while computing a linear predictor for regression. 

We use the algebra data, a benchmark data from KDD Cup 2010~\cite{Algebra-Data}, for the following experiments. Specifically we use a preprocessed version of the algebra data\footnote{https://www.csie.ntu.edu.tw/~cjlin/libsvmtools/datasets/} for our simulations. This preprocessed data set consists of $8,407,752$ samples from two classes, and each of the samples is represented as a feature feature of dimensionality $20,216,830$. In our experiments, we set $\alpha = 1$ and $\beta = 0.5$ for Eq.~(\ref{eq:lasso}). We tune the initial step size respectively for each of the competing algorithms to get the (nearly) best performance; specifically in this experiments, we set $\eta_0 = 0.5$ for SGD, $\eta_0 = 0.3$ and $\lambda = 0.03$ for Epro, $\eta_0 = 0.1$ for OneProj, and $\eta_0 = 0.1$ for LogT.

In the experiments, we respectively run all competing algorithms for $2000$ iterations; we then record the obtained objective values and the corresponding computation time. The experimental results are presented in Figure~\ref{fig:lasso}. The left plot shows how the objective value is changed with respect to the algorithm iteration. Note that for Eq.~(\ref{eq:lasso}), the number of algorithm iterations is equal to the number of stochastic gradient computation (the access to the subgradient of the objective function).  From this plot, we can observe that after running $2000$ iterations, Epro and SGD attain smaller objective values, compared to OneProj and LogT; we can also observe that OneProj empirically converges slightly faster than logT. The right plot shows how the objective value is changed with respect to the computation time. For this experiment, we set the maximum computation time to 3035 seconds, which is the computation time required by running Epro for $2000$ iterations. We can observe that Epro attain a smaller objective value, compared to the other three competing method; meanwhile, the standard SGD and OneProj attain similar objective values, given a fixed amount of computation time.

\subsection{EXPERIMENTS ON THE LMNN FORMULATION}

We apply the four competing algorithms to solve the LMNN formulation in Eq.~(\ref{eq:lmnn-sparse}). We use the Cora data~\cite{CoraData} for the following experiments. Cora consists of $2708$ scientific publications exclusively from $7$ different categories. Each publication is represented by a normalized vector of length $1$ and dimensionality $1433$. From this data, we construct $5416$ neighbor pairs (NP) by randomly selecting $2$ publications of the same label; we then construct $16248$ non-neighbor (NNT) by randomly selecting 3 non-neighbor publications (of a different label) for each of the NPs. Therefore, in each iteration of the SGD-type methods, we can use a NP and a NNT to construct a stochastic gradient for the optimization formulation. We set $c = 0.5$, $\mu_1 = 10^{-4}$, and $\mu_2 = 10^{-3}$ in Eq.~(\ref{eq:lmnn-sparse}). We terminate the algorithms when the iteration number is larger than $4,000$ or the relative change of the objective values in two iterations is smaller than $10^{-8}$; we also record the obtained objective values, the required iteration number, and the computation time. Similarly we tune the initial step size respectively for the competing algorithms; specifically, we set $\eta_0 = 4 \times 10^{-8}$ for SGD, $\eta_0 = 10^{-5}$ and $\lambda = 0.1$ for Epro, $\eta_0 = 5 \times 10^{-7}$ for OneProj, and $\eta_0 = 10^{-6}$ for LogT.
\begin{figure}
\includegraphics[width=0.24\textwidth,height=3cm]{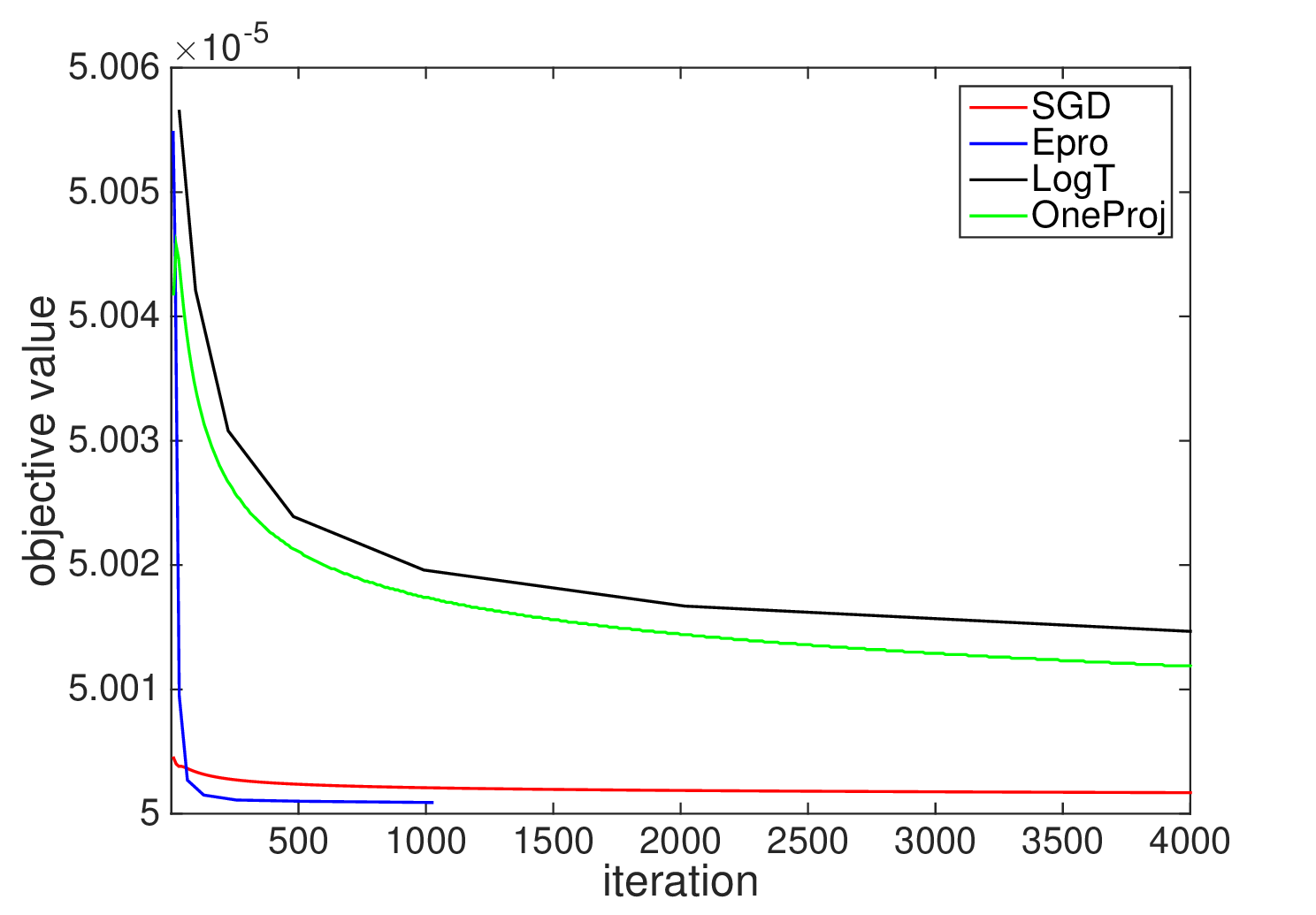} \hskip -0.1in
\includegraphics[width=0.24\textwidth,height=3cm]{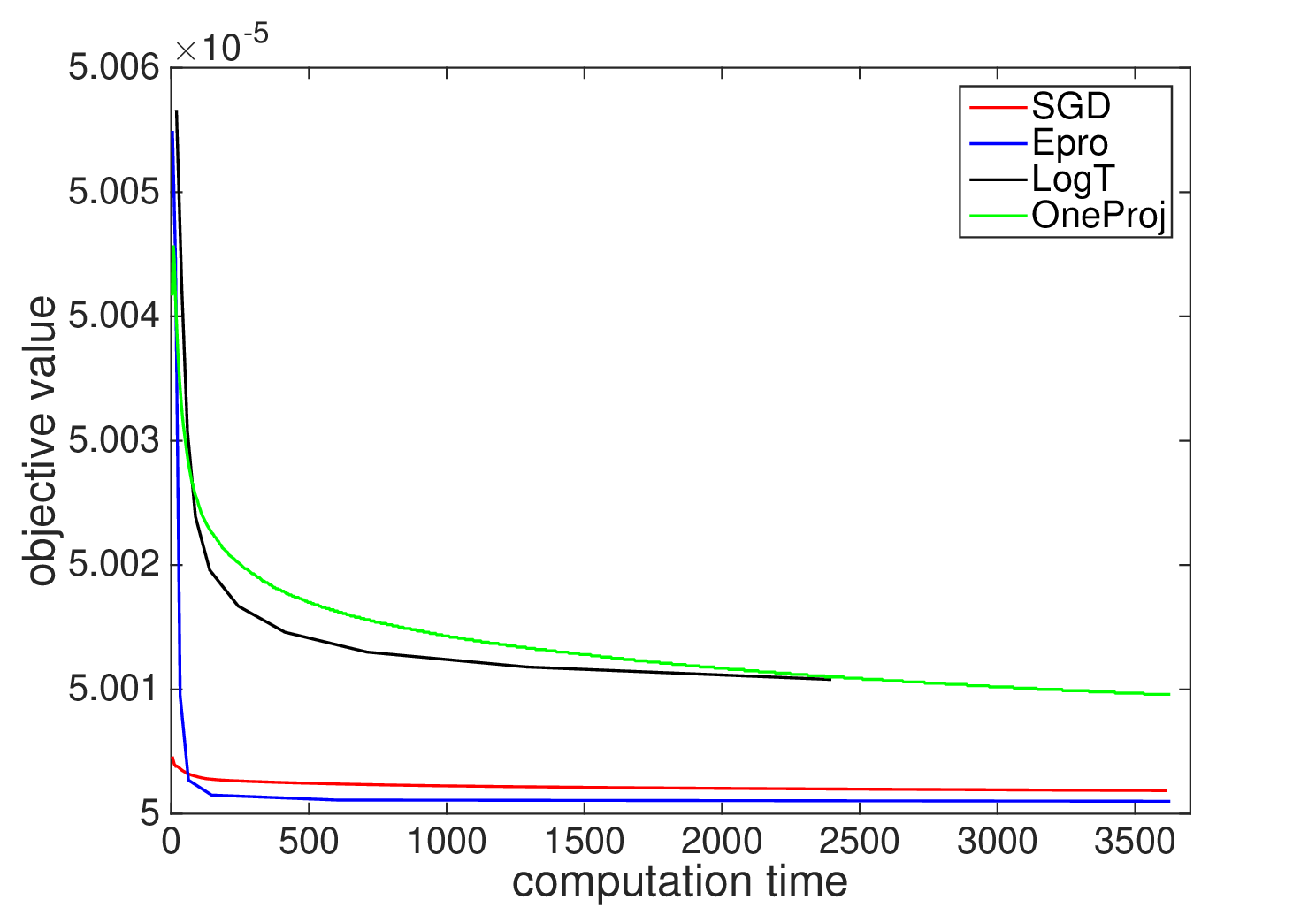}
\caption{Empirical comparison of the four competing methods for solving Eq.~(\ref{eq:lmnn-sparse}). (1) Left plot: the change of the objective values with respect to the iteration number. (2) Right plot: the change of the objective value with respect to the computation time.}
\label{fig:lmnn}       
\end{figure}

The experimental results are presented in Figure~\ref{fig:lmnn}. Similarly in the left plot, we illustrate how the objective value changes with respect to the iteration number. For the LMNN formulation in Eq.~(\ref{eq:lmnn-sparse}), we can observe that Epro converges empirically much faster than all three competing algorithms; in particular, Epro converges after $1024$ iterations, while the other $3$ algorithms need more iterations. In the right plot, we illustrate how the objective value changes with respect to the computation time. We can observe that Epro converges using a smaller amount of computation time. specifically, in our experiment Epro converges with the computation time as $3622$ seconds; while the other three competing algorithms need much more computation time.

\section{CONCLUSIONS}
We proposed an epoch-projection based SGD method, called Epro-SGD, for stochastic strongly convex optimization. The proposed Epro-SGD applies SGD on each iteration within its epochs and only performs a projection at the end of each epoch. Our analysis shows that Epro-SGD requires only a logarithmic number of projections, while achieves a guaranteed optimal rate of convergence both in expectation as well as with high probability. Additionally we proposed a variant of Epro-SGD based on an existing dual averaging method, called Epro-ORDA, which exploit structures of the optimization problems by incorporating an associated proximal mapping iteratively. For illustration, we applied the proposed Epro-SGD method for solving a large margin distance metric learning formulation and a constrained Lasso formulation respectively with a positive definite constraint. Our empirical results demonstrate the effectiveness of the proposed method. 

\section*{Acknowledgements}
\vspace*{-0.1in}
The authors would like to thank the anonymous reviewers for their helpful  comments. T. Yang was supported in part by NSF (1463988, 1545995).

{\small
\bibliographystyle{style/spmpsci}      
\bibliography{EpoSGD.bib}
}

\newpage

\section*{SUPPLEMENTAL MATERIALS}

\begin{lem}~\cite{DBLP:journals/jmlr/HazanK11a} \label{lem:3}
If $f(\x)$ is $\beta$-strongly convex and $\x_*$ denotes the optimal solution to $\min_{\x\in\D}f(\x)$. 
For any $\x\in\D$, we have $f(\x)- f(\x_*)\leq 2 G_1^2 / \beta$. 
\end{lem}
\begin{proof}
From Assumption A1, we have $\|\partial f(\x)\|_2\leq G_1$. Hence 
\[
f(\x) - f(\x_*)\leq G_1 \|\x - \x_*\|_2.
\]
Moreover from the strong convexity in $f(\cdot)$ we have
\[
f(\x) - f(\x_*)\geq \frac{\beta}{2}\|\x - \x_*\|_2^2.
\]
From the two inequalities above, we can easily verify that
\[
\|\x - \x_*\|_2 \leq \frac{2G_1}{\beta}, \,\,\, f(\x) - f(\x_*)\leq \frac{2G_1^2}{\beta}. 
\]
This completes the proof.
\end{proof}

\subsection*{Proof of Theorem~\ref{cor:1}}
The proof of Theorem~\ref{cor:1} is based on an important result, as summarized in Lemma~\ref{lem:martingle}.
\begin{lem}\cite{DBLP:conf/nips/MahdaviYJZY12}\label{lem:martingle}
Assume $\|\x_* - \x_t\|_2\leq D$ for all $t$.  Define $D_T= \sum_{t=1}^{T} \|\x_t - \x\|_2^2$ and $\Lambda_T =\sum_{t=1}^T\zeta_t(\x)$. We have
\begin{eqnarray*}
 \Pr\left( \Lambda_T \leq 4G_1\sqrt{D_T\ln\frac{m}{\epsilon}} + 2G_1D\ln\frac{m}{\epsilon}\right) \\
  + \Pr\left( D_T \leq \frac{D^2}{T} \right) \geq 1 - \epsilon,
\end{eqnarray*}
where $m = \lceil2\log_2 T \rceil$ and $\sum_{t=1}^T\zeta_t(\x)=\sum_{t=1}^{T} (\nabla f(\x_t) - \g(\x_t))^{\top}(\x-\x_t)$.
\end{lem}
{\noindent \it Proof of Theorem~\ref{cor:1}} 
The proof below follows from techniques used in Lemma~\ref{lem:1} and Theorem~\ref{thm:1}. Since $F(\x)$ is $\beta$-strongly convex, we have
\[
F(\x_t) - F(\x) \leq (\x_t - \x)^{\top}\nabla F(\x_t) - \frac{\beta}{2}\|\x- \x_t\|_2^2.
\]
Combining the above inequality with  the inequality in~(\ref{eqn:basic}) and taking summation over all $t=1,\ldots, T$, we have
\begin{eqnarray}\label{eqn:high}
\sum_{t=1}^T(F(\x_t) - F(\x)) & \leq & \underbrace{\frac{\|\x_1 - \x\|_2^2}{2\eta} + \eta T(G_1^2 + \lambda^2G_2^2)}\limits_{BT}  \nonumber \\
                                             &        & + \sum_{t=1}^T\zeta_t(\x) - \frac{\beta}{2}D_T.
\end{eqnarray}
We substitute the bound in Lemma~\ref{lem:martingle} into the above inequality with $\x=\x^*$. We consider two cases. In the first case, we assume $D_T\leq D^2/T$. As a result, we have
\begin{eqnarray*}
\sum_{t=1}^T\zeta_t(\x^*) & = & \sum_{t=1}^{T} (\nabla f(\x_t) - \g(\x_t))^{\top}(\x^*-\x_t) \\
                                         & \leq & 2G_1\sqrt{TD_T} \leq 2G_1D,
\end{eqnarray*}
which together with the inequality in~(\ref{eqn:high}) leads to the bound
\[
    \sum_{t=1}^T (F(\x_t) - F(\x^*))\leq 2G_1D + BT.
\]
In the second case, we assume
\begin{eqnarray*}
 \sum_{t=1}^T\zeta_t(\x^*) & \leq & 4G_1\sqrt{D_T\ln\frac{m}{\epsilon}} + 4G_1\ln\frac{m}{\epsilon} \\ 
                                          & \leq & \frac{\beta}{2}D_T + \left(\frac{8G_1^2}{\beta} + 4G_1\right)\ln\frac{m}{\epsilon},
\end{eqnarray*}
where the last step uses the fact $2\sqrt{ab} \leq a^2 + b^2$. We thus have
\[
 \sum_{t=1}^T(F(\x_t) - F(\x^*)) \leq   \left(\frac{8G_1^2}{\beta} + 2G_1D\right)\ln\frac{m}{\epsilon} +  BT
\]
Combing the results of the two cases, we have, with a probability $1 - \epsilon$,
\begin{eqnarray*}
\sum_{t=1}^T (F(\x_t) - F(\x^*)) & \leq &  \left(\frac{8G_1^2}{\beta} + 2G_1D\right)\ln\frac{m}{\epsilon} \\
                                                 &        & + 2G_1D + BT,
\end{eqnarray*}
where $C =  \left(\frac{8G_1^2}{\beta} + 2G_1D\right)\ln\frac{m}{\epsilon}  + 2G_1D$. 
Following the same analysis, we have
\begin{align*}
f(\xt_T) - f(\x_*)\leq \frac{\mu C}{T} + \frac{\mu\|\x_1 - \x_*\|_2^2}{2\eta T} + \mu\eta G^2
\end{align*}
Let $\Delta_k = f(\x^1_k) - f(\x_*)$. By  induction, we have
\begin{align*}
\Delta_{k+1}\leq \frac{\mu C}{T_k} + \frac{\mu \Delta_k}{2\eta_k T_k \beta} + \mu \eta_k G^2
\end{align*}
Assume $\Delta_k\leq V_k\triangleq \frac{\mu^2G^2}{2^{k-2}\beta}$,  by plugging the values of $\eta_k, T_k$, we have
\begin{align*}
\Delta_{k+1}\leq \frac{V_k}{6} + \frac{V_k}{6} + \frac{V_k}{6}= \frac{V_k}{2} = V_{k+1}
\end{align*}
where we use $T_1\geq \max\left(\frac{3C\beta}{\mu G^2}, 9\right)$ and $T_k \geq \max\left(\frac{6\mu c}{V_k}, \frac{18\mu^2G^2}{V_k\beta}\right)$ and $\eta_k = \frac{V_k}{6\mu G^2} = \frac{2\mu}{ 2^k(3\beta)}$. This completes the proof of this theorem.

\subsection*{Proof of Lemma~\ref{lem:prix}}
To prove Lemma~\ref{lem:prix}, we derive an inequality similar to~Eq.~(\ref{eqn:key}); the rest proof of Lemma~\ref{lem:prix} is similar to that of Lemma~\ref{lem:1}. 
\begin{cor} \label{cor:000}
Given a $\beta$-strongly convex function $\widehat f(\x) =f(\x) + g(\x)$, and a sequence $\{\x_t\}$ defined by the update $
\x_{t+1} = \min_{\x}\frac{1}{2}\|\x-(\x_t -\eta\g(\x_t) ) \|_2^2 + \eta g(\x)$. Then for any $\x$, we have
\begin{eqnarray*}
&        &  \sum_{t=1}^T\left[f(\x_t) + g(\x_{t+1}) - f(\x) - g(\x)\right]  \\
& \leq  & \frac{\|\x-\x_1\|_2^2}{2\eta} + \frac{\eta}{2}\sum_{t=1}^T\|\g(\x_t)\|_2^2  
+ \sum_{t=1}^T(\x-\x_t)^{\top}(\g(\x_t) \\ 
&         & - \nabla f(\x_t)) 
 - \frac{\beta}{2}\sum_{t=1}^T\|\x-\x_{t+1}\|_2^2. 
\end{eqnarray*}
\end{cor}
Corollary~\ref{cor:000} can be proved using techniques similar to the ones in~\cite{DBLP:conf/colt/DuchiSST10} but with extra care on the stochastic gradient.  
As a consequence we have
\begin{eqnarray*}
&        & \frac{1}{T}\E\left[\sum_{t=1}^T\hat f(\x_t)- \hat f(\x)\right] \\
& \leq & \frac{\E[\|\x-\x_1\|_2^2]}{2\eta T}  + \eta( G_1^2 + \lambda G_2^2) + \frac{g(\x_{1}) -g(\x_{T+1})}{T}
\end{eqnarray*}

\subsection*{Proof of Lemma~\ref{lem:ORDA}}
The lemma is a corollary of results in~\cite{NIPS2012_4543} for general convex optimization. In particular, if we consider the stochastic composite optimization 
\begin{align*}
F(\x) = \phi(\x) + g(\x)
\end{align*}
where $g(\x)$ is a simple function such that its proximal mapping can be easily solved and $\phi(\x)$ is only accessible through a stochastic oracle that returns a stochastic subgradient $\g(\x)$. To state the convergence of ORDA for general convex problems, \cite{NIPS2012_4543} makes the following assumptions: (i) $\E[\|\g(\x) - \E\g(\x)\|^2_2]\leq \sigma^2$  and (ii) 
\[
\phi(\y) - \phi(\x) - (\y - \x)^{\top}\partial \phi(\x)\leq M\|\y - \x\|_2
\]
When $\|\partial \phi(\x)\|_2\leq G$, the first inequality holds  $\sigma=G$ and the second inequality holds  with $M=2G$. Applying to the augmented objective 
\begin{align*}
F(\x) = f(\x) + \lambda [c(\x)]_+ + g(\x)
\end{align*}
We note that $\sigma = G_1$ and $M=2(G_1 + \lambda G_2)$. Follow the inequality (26) in the appendix of \cite{NIPS2012_4543}, we obtain that 
\[
\E[F(\x_{T+2}) - F(\x_*)]\leq \frac{4\|\x_1 - \x_*\|_2^2}{\eta \sqrt{T}} + \frac{2\eta(\sigma + M)^2}{\sqrt{T}}
\]
by using the Euclidean distance $V(\x, \y) = \frac{1}{2}\|\x - \y\|_2^2$ and their notation $\tau=1$, and noting that $\eta$ is the inverse of their notation $c$. Then the second inequality is Lemma~\ref{lem:ORDA} can be proved similarly as for Lemma~\ref{lem:1}. 

\subsection*{Proof of Theorem~\ref{thm:2}}
\begin{proof}
Recall $\mu = \rho/(\rho - G_1/\lambda)$ and $G = 3G_1 + 2\lambda G_2$. Let $V_k = \left( \mu^2 G^2\right) / \left( 2^{k-2}\beta \right)$. By the values of $\eta_k$ and $T_k$ we have
\begin{equation*}
T_k   = 2^{k+3}= \frac{32\mu^2G^2}{V_k\beta} , 
\eta_k  = \frac{\mu}{2^{(k-1)/2}\beta}
=\frac{V_k \sqrt{T_k}}{8\mu G^2}. 
\end{equation*} 
Define $\Delta_k = \hat f(\x^k_1) - \hat f(\x_*)$. We first prove the inequality
\begin{align*}
\E[\Delta_k]\leq V_k
\end{align*}
by induction. It is true for $k=1$ because of Lemma~\ref{lem:3}, $\mu>1$ and $G^2>G_1^2$. Now assume it is true for $k$ and we prove it for $k+1$. For a random variable $X$ measurable with respect to the randomness up to epoch $k+1$. Let $\E_k[X]$ denote the expectation conditioned on all the randomness up to epoch $k$. Following Lemma~\ref{lem:1}, we have
\begin{eqnarray}
\E_k[\Delta_{k+1}] \leq \mu \left[\frac{2\eta_k G^2}{\sqrt{T_k}}+ \frac{\E[4\|\x^k_1 - \x_*\|_2^2]}{\eta_k \sqrt{T_k}}\right]
\end{eqnarray}
Since $\Delta_k = f(\x^k_1) - f(\x_*)\geq \beta\|\x_1^k-\x_*\|_2^2/2$ by the strong convexity, we have
\begin{eqnarray*}
\E[\Delta_{k+1}] & \leq & \mu \left[\frac{2\eta_k G^2}{\sqrt{T_k}}+ \frac{\E[8\Delta_k]}{\eta_k \sqrt{T_k}\beta}\right] \nonumber \\
& = & \frac{2\eta_k \mu G^2}{\sqrt{T_k}} + \frac{V_k\mu}{\eta_k \sqrt{T_k}\beta} = \frac{V_k}{4} + \frac{V_k}{4} = \frac{V_k}{2}
\end{eqnarray*}
where we use the fact $\eta_k/\sqrt{T_k} = V_k /(8\mu G^2) $ and $T_k  = 32\mu^2 G^2/(V_k\beta) $. Thus,  we get
\[
\E[f(\x_1^{k^\dagger+1})] -f(\x_*) = \E[\Delta_{k^\dagger+1}] \leq V_{k^\dagger+1} = \frac{\mu^2 G^2}{2^{k^\dagger-1}\beta}
\]
Note that the total number of epochs satisfies
\[
\sum_{k=1}^{k^\dagger}(T_k+1) = 16(2^{k^\dagger}-1) + k^\dagger\leq T
\]
By some reformulations, we complete the proof of this theorem.
\end{proof}

\subsection*{Proof of Lemma~\ref{lem:martingle}}

The proof of Lemma~\ref{lem:martingle} is based on {\it the Bernstein Inequality for Martingales}~\cite{DBLP:conf/ac/BoucheronLB03}. We present its main result below for completeness.
\begin{thm}~[Bernstein Inequality for Martingales]  \label{thm:bernstein} 
Let $X_1, \ldots , X_n$ be a bounded martingale difference sequence with respect to the filtration $\F = (\F_i)_{1\leq i\leq n}$ and with $\|X_i\| \leq K$. Let
\begin{equation*}
S_i = \sum_{j=1}^i X_j
\end{equation*}
be the associated martingale. Denote the sum of the conditional variances by
\begin{equation*}
    \Sigma_n^2 = \sum_{t=1}^n \E\left[X_t^2|\F_{t-1}\right],
\end{equation*}
Then for all constants $t$, $\nu > 0$,
\begin{equation*}
\Pr\left[ \max\limits_{i=1, \ldots, n} S_i > t \mbox{ and } \Sigma_n^2 \leq \nu \right] \leq \exp\left(-\frac{t^2}{2(\nu + Kt/3)} \right),
\end{equation*}
and therefore,
\begin{equation*}
    \Pr\left[ \max\limits_{i=1,\ldots, n} S_i > \sqrt{2\nu t} + \frac{\sqrt{2}}{3}Kt \mbox{ and } \Sigma_n^2 \leq \nu \right] \leq e^{-t}.
\end{equation*}
\end{thm}
{\noindent \it Proof of Lemma~\ref{lem:martingle}.}
Define martingale difference $X_t = (\x- \x_t)^{\top}(\nabla f(\x_t) - \g(\x_t))$ and martingale $\Lambda_T= \sum_{t=1}^{T} X_t$. Define the conditional variance $\Sigma_T^2$ as
\[
    \Sigma_T^2 = \sum_{t=1}^{T} \E_{\xi_t}\left[X_t^2 \right] \leq 4G_1^2\sum_{t=1}^{T} \|\x_t - \x\|_2^2 = 4G_1^2D_T.
\]
Define $K = 2G_1D$. Thus, $\|X_t\|_2\leq K$. We have
\begin{eqnarray*}
&                            & \hskip -0.15in \Pr\left(\Lambda_T\geq 2\sqrt{4G_1^2D_T \tau} + \sqrt{2}K\tau/3\right) \\
& \hskip -0.2in =    & \hskip -0.15in \Pr\left(\Lambda_T\geq 2\sqrt{4G_1^2D_T \tau} + \sqrt{2}K\tau/3, \Sigma_T^2 \leq 4G_1^2D_T\right) \\
& \hskip -0.2in =    & \hskip -0.15in \Pr\left(\Lambda_T\geq 2\sqrt{4G_1^2D_T\tau} + \sqrt{2}K\tau/3, \Sigma_T^2 \leq 4G_1^2D_T, \right. \\ 
&                            & \hskip -0.15in \left. D_T \leq \frac{D^2}{T} \right) + \sum_{i=1}^m \Pr\left(\Lambda_T \geq 2\sqrt{4G_1^2 D_T \tau} + \right. \\
&                            & \hskip -0.15in \left. \sqrt{2}K\tau/3, \Sigma_T^2 \leq 4G_1^2D_T, \frac{D^2}{T}2^{i-1} < D_T  \leq \frac{D^2}{T} 2^{i} \right) \\
& \hskip -0.2in \leq & \hskip -0.15in \Pr\left(D_T \leq \frac{D^2}{T}\right) + \sum_{i=1}^m \Pr\left(\Lambda_T \geq \sqrt{2\times 4G_1^2\frac{D^2}{T} 2^{i}\tau} \right. \\
&                             & \hskip -0.15in + \left.  \sqrt{2}K\tau/3, \Sigma_T^2 \leq 4G_1^2\frac{D^2}{T}2^i\right) \\
& \hskip -0.2in \leq  & \hskip -0.15in \Pr\left(D_T \leq \frac{D^2}{T}\right) + me^{-\tau},
\end{eqnarray*}
where we use the fact $\|\x_t - \x\|_2^2 \leq D^2$ for  all $t$ and $m = \lceil2\log_2 T \rceil$, and the last step follows the Bernstein inequality for martingales. We complete the proof by setting $\tau= \ln(m/\epsilon)$.


\end{document}